\documentclass[10pt, a4paper]{article}

\usepackage{lrec-coling2024}
\usepackage{times}
\usepackage{latexsym}

\usepackage[T1]{fontenc}

\usepackage[utf8]{inputenc}

\usepackage{microtype}

\usepackage{inconsolata}
\usepackage{graphicx}
\usepackage{multirow}
\usepackage{amsmath}
\usepackage{svg}
\usepackage{tikz}
\usepackage{amssymb}
\usepackage{mathtools}
\usepackage{amsthm}
\usepackage{pythonhighlight}
\theoremstyle{plain}
\newtheorem{theorem}{Theorem}[section]

\theoremstyle{definition}

\theoremstyle{remark}

\usepackage{relsize}
\usepackage{amsfonts}
\usepackage{subfigure}
\usepackage{tikz}
\usepackage{bm}
\usetikzlibrary{shapes.geometric}
\usepackage{pgfplots}
\usepackage{booktabs}
\usepackage{colortbl}
\definecolor{mygray}{gray}{.92}
\newcommand*\samethanks[1][\value{footnote}]{\footnotemark[#1]}

\title{The Open-World Lottery Ticket Hypothesis for OOD Intent Classification}

\name{Yunhua Zhou\textsuperscript{$1,2$}\sthanks{\ \  Equal contribution.}, 
    Pengyu Wang\textsuperscript{$1$}\samethanks,
    Peiju Liu\textsuperscript{$1$}, 
    Yuxin Wang\textsuperscript{$1$},
    Xipeng Qiu\textsuperscript{$1$}\sthanks{\ \  Corresponding author.}
}

\address{
\textsuperscript{$1$}School of Computer Science, Fudan University \quad
\textsuperscript{$2$}Shanghai AI Laboratory \\
\texttt{zhouyunhua@pjlab.org.cn} \quad
\texttt{xpqiu@fudan.edu.cn} \\
\texttt{\{pywang22, pjliu21, wangyuxin21\}@m.fudan.edu.cn}
}

\abstract{
Most existing methods of Out-of-Domain (OOD) intent classification rely on extensive auxiliary OOD corpora or specific training paradigms. However, they are underdeveloped in the underlying principle that the models should have differentiated confidence in In- and Out-of-domain intent. In this work, we shed light on the fundamental cause of model overconfidence on OOD and demonstrate that calibrated subnetworks can be uncovered by pruning the overparameterized model. Calibrated confidence provided by the subnetwork can better distinguish In- and Out-of-domain, which can be a benefit for almost all~\textit{post hoc} methods.
In addition to bringing fundamental insights, we also extend the Lottery Ticket Hypothesis to open-world scenarios. We conduct extensive experiments on four real-world datasets to demonstrate our approach can establish consistent improvements compared with a suite of competitive baselines. 
 \\ \newline  }

\begin{document}

\maketitleabstract

\section{Introduction}
Interactive Systems, such as Task-Oriented Dialog Systems(TODS), are gradually integrating into and facilitating the daily life of people. However, in open-world scenarios, i.e., the training and test set come from the different distributions or domains, it is often encountered that the expressed intents are reasonable but beyond the domains supported by the Interactive Systems, resulting in mapping the intent to the wrong subsequent processing pipelines. Therefore, Interactive Systems not only need to maintain performance in In-Domain (IND) intents but also need to correctly identify Out-of-Domain (OOD) intents.

\begin{figure}[th]
    \centering
    \subfigure[\scriptsize{Reliability Diagrams (Original)}]{
    \includegraphics[width=0.48\linewidth]{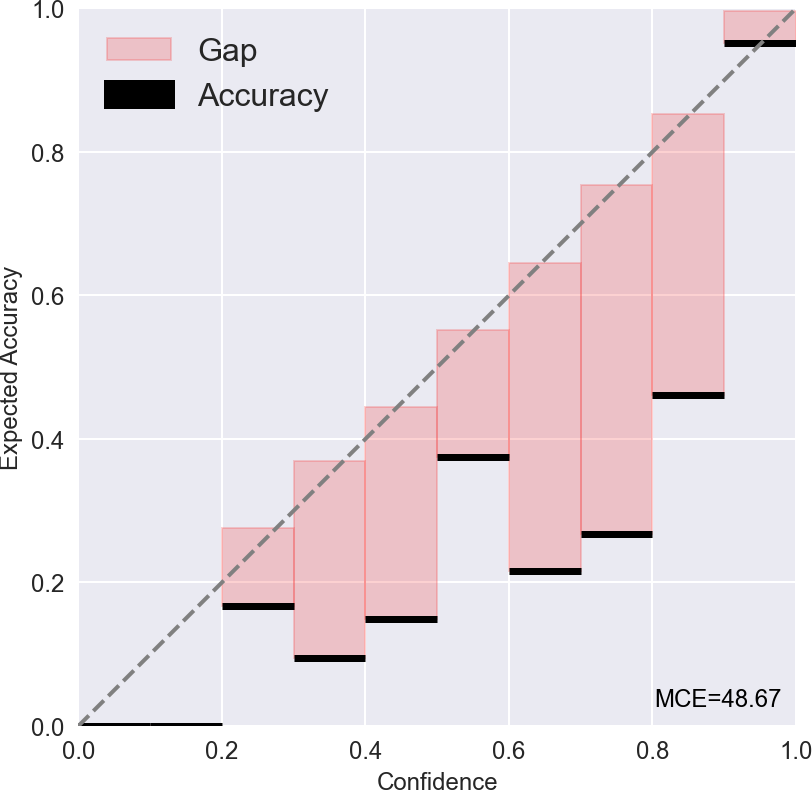}\label{r-ori}}
    \subfigure[\scriptsize{Reliability Diagrams (\textbf{OLT})}]{
    \includegraphics[width=0.48\linewidth]{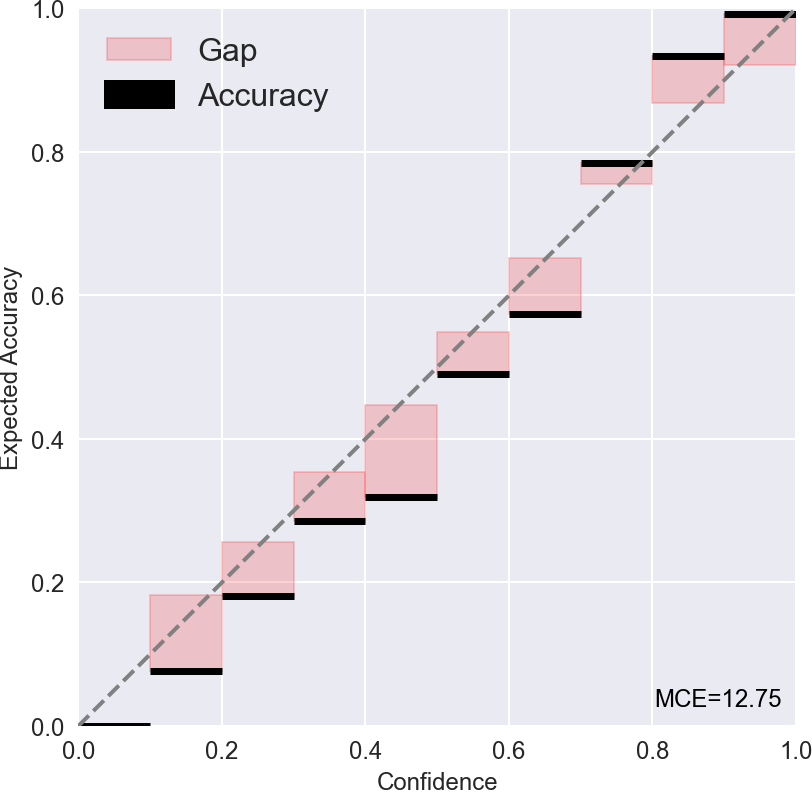}\label{r-lottery}}
    \quad
    \subfigure[\scriptsize{Distribution of Scores (Original)}]{
    \includegraphics[height=0.45\linewidth,width=0.48\linewidth]{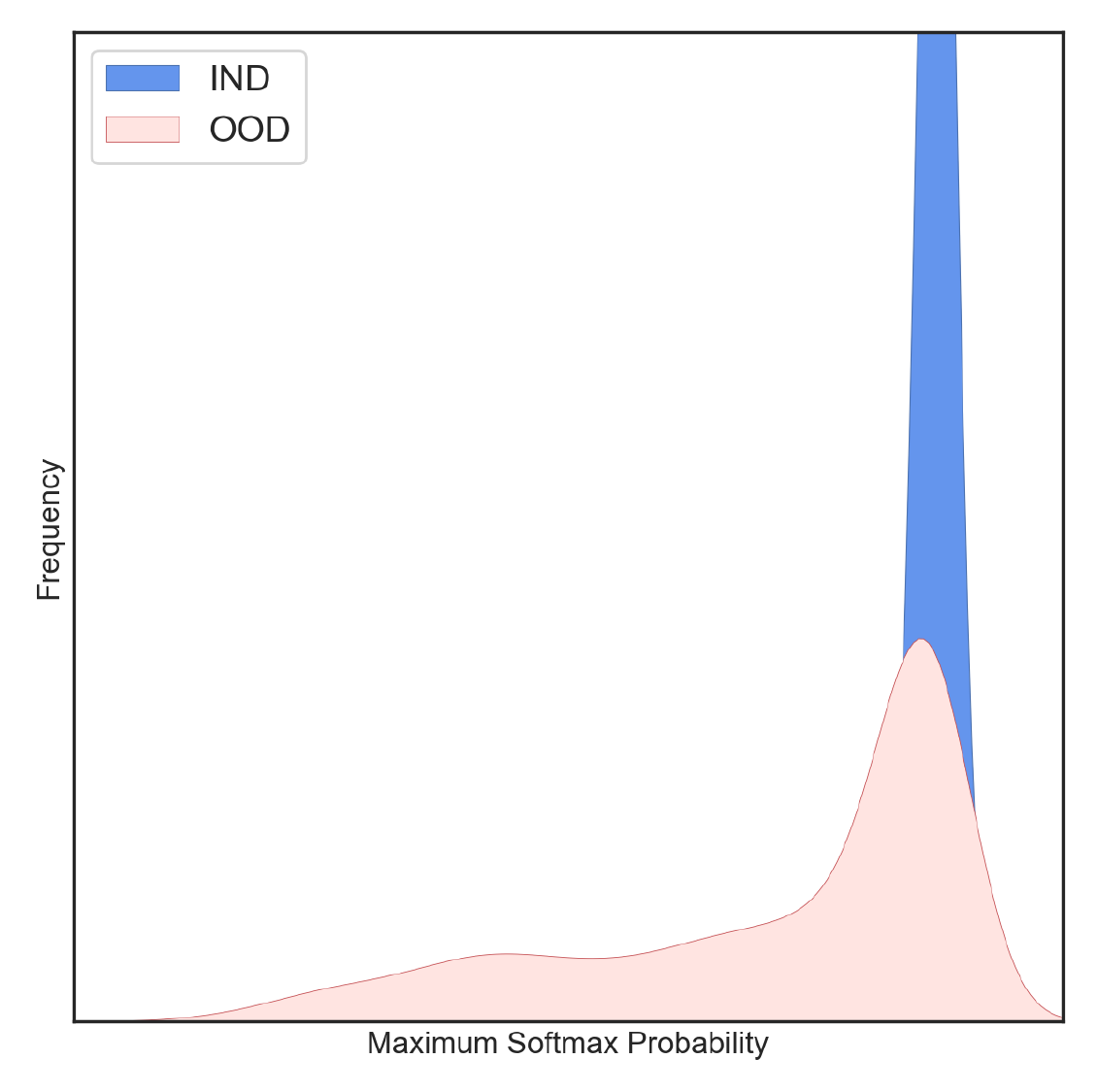}\label{d-ori}}
    \subfigure[\scriptsize{Distribution of Scores (\textbf{OLT})}]{
    \includegraphics[height=0.45\linewidth,width=0.48\linewidth]{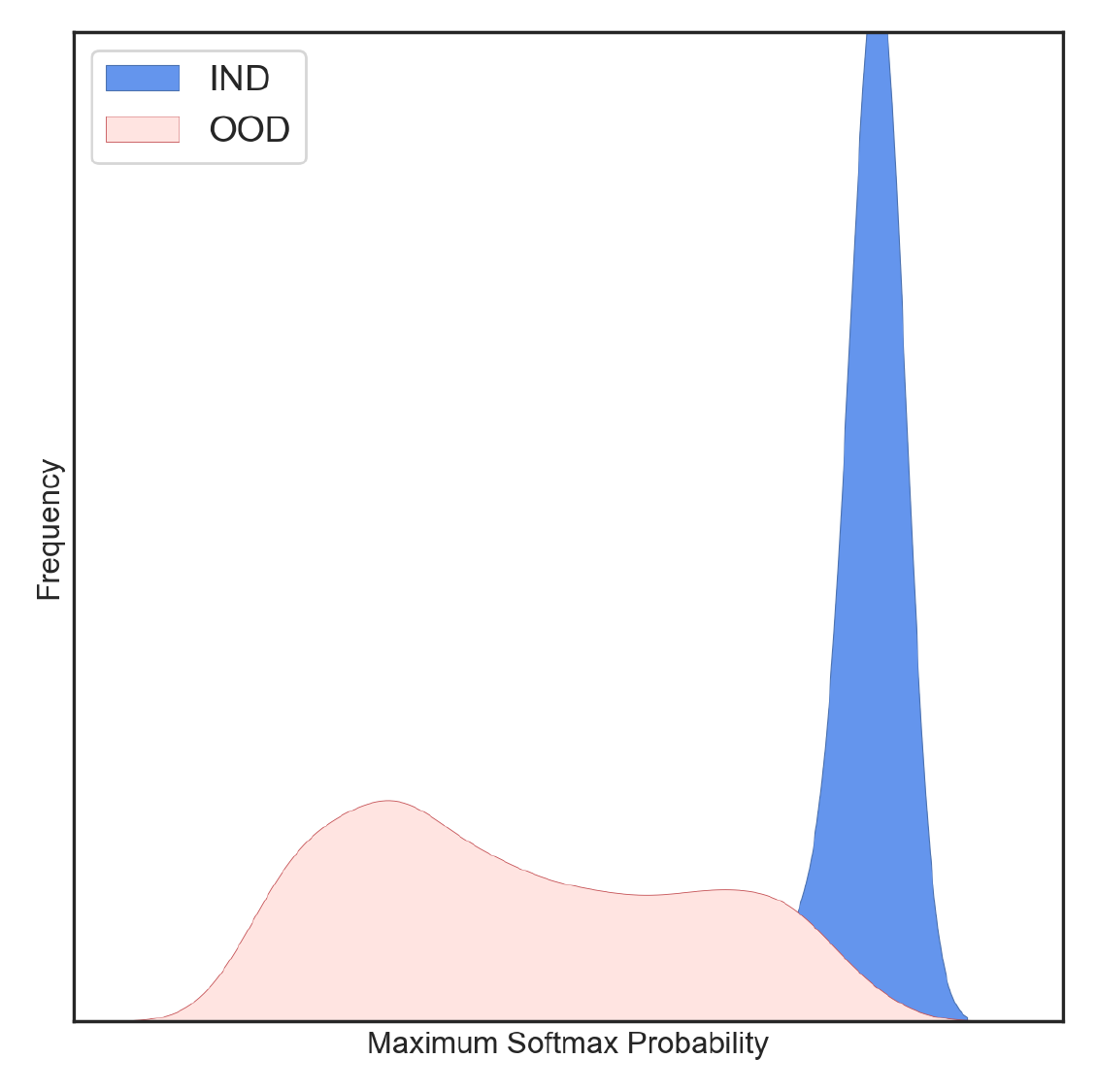}\label{d-lottery}}
    \caption{Plots showing (\textbf{Top}) Reliability diagrams and (\textbf{Bottom}) The distribution of In-and Out-of-domain uncertainty scores in the Stackoverflow dataset. The OLT denotes our proposed Open-world Lottery Ticket. The reliability diagrams (pink) are about the function of confidence, which measures the gap (i.e., miscalibration) between expected sample accuracy (black) and confidence. The Maximum Calibration Error (MCE) measures the maximum gap. If a model meets perfect calibration, the gap is zero and the diagrams disappear.}
    \label{fig:r-distribution}
\end{figure}

Recently, to get critical insights into \textit{Does the model knows what it does not know?} i.e., the model should be high-confident in IND and low-confident in OOD (due to unseen),~\citet{hendrycks-etal-2020-pretrained} take a step to show that compared with previous models, such as LSTMs, the confidence scores produced by the Pre-Trained Models'maximum softmax probabilities can significantly distinguish IND and OOD but remain a long way before it is perfect. 

What prevents the confidence of the model from being further trusted? 
Current efforts have primarily concentrated on developing appropriate~\textit{post hoc} (i.e., not involved in training and after training) methods or scoring functions based on maximum softmax probability, such as MSP~\cite{hendrycks-etal-2020-pretrained}, Entropy~\cite{DBLP:journals/corr/abs-2010-03759}, to measure OOD uncertainty. Despite the advancements made, these approaches are inherently limited and lack broad applicability, as the root underlying causes have yet to be thoroughly investigated and understood.

We take a step forward and observe that the maximum softmax probability outputted by the overparameterized model
cannot correctly reflect the confidence of the model, which is known as poor-calibrated~\citep{DBLP:journals/corr/GuoPSW17} and can be visualized~\footnote{ 
\url{https://github.com/hollance/reliability-diagrams}}
by \textbf{reliability diagrams} as shown in Fig.\ref{r-ori}. When encountering open-world scenarios, the unreliable predicted confidence (and other ~\textit{post hoc} measures based on it) given by the poor-calibrated model cannot be measured to uncertain about samples correctly. Furthermore, subsequent analysis (Section~\ref{overconfident-model}) shows that the overparameterized model tends to be overconfident, which is also consistent with the experiment as shown in Fig.\ref{d-ori}. This phenomenon undermines the underlying idea that the model should be much low-confident in OOD and makes it non-trivial to distinguish between IND and OOD. 

In this paper, in addition to giving fundamental insight, we also explore how to calibrate the model to provide reliable confidence. To this end, we first set out to establish the effect of overparameterization in poor calibration and theoretically demonstrate overparameterization would aggravate overconfident predictions on OOD inputs.
Inspired by this, different from the previous work, we do not design a specific simple method to measure OOD uncertainty. Instead, through masking the parameters that are not of interest to the target task,
we prune a calibrated subnetwork from an overparameterized Pre-Trained model during training, which has more general reliable confidence to better differentiate IND and OOD and can be a benefit for almost all~\textit{post hoc} methods.

Especially, beyond the established awareness that temperature scaling can help improve calibration in the~\textit{post hoc} phase empirically~\citet{DBLP:journals/corr/GuoPSW17}, we contribute a new general insight on temperature scaling in open-world scenarios and theoretically demonstrate temperature scaling can substantially differentiate IND and OOD.

Going further, combined with the above calibration of subnetwork and~\textit{post hoc} measure (temperature scaling adopted in this paper), we can further generalize the Lottery Ticket Hypothesis~\cite{DBLP:conf/iclr/FrankleC19} to the open-world. The \textbf{O}pen-world \textbf{L}ottery \textbf{T}icket \textbf{H}ypothesis (OLTH) is articulated as:

\textsl{An initialized overparameterized neural network contains a winning subnetwork---through one-shot pruning and minor post-processing, which can match the commensurate performance in IND identification as original, but also better detect OOD at a commensurate training cost as the original.}

Compared with the original Lottery Ticket Hypothesis, we generate \textbf{O}pen-world \textbf{L}ottery \textbf{T}icket (OLT) through one-shot pruning without iterative pruning. The OLT could be better-calibrated, as shown in Fig.\ref{r-lottery} and not only guarantees the precision of IND recognition but can better distinguish between IND and OOD, as shown in Fig.\ref{d-lottery}, signifying its adaptability to the open-world. Extensive experiments are conducted on four real-world datasets and further verify our hypothesis. Our contributions and insights are as: \\
\textbf{(Theory)} We establish the effect of overparameterization in overconfidence and demonstrate that the well-calibrated confidence of the subnetwork can help improve OOD detection. Furthermore, we empirically extend the LTH--we can identify a lottery ticket from the overparameterized model that is more suitable for the open-world setting.\\
\textbf{(Methodology)} 
We propose a one-shot (without Iterative) Magnitude Pruning to uncover the lottery ticket of interest to the target task, which has more general reliable confidence to better differentiate IND and OOD and can be a benefit for almost all~\textit{post hoc} uncertainty measurements. \\
\textbf{(Experiments)} Extensive experiments and analysis show that our method can improve OOD detection on the premise of the accuracy of IND recognition, which confirms the correctness of the Open-world Lottery Ticket Hypothesis.\footnote{Codes is publicly available at: \url{https://github.com/zyh190507/Open-world-Lottery}}

\section{Related Work}
\label{related-work}
There are two types of work close to our research--Out-of-domain Detection and Sparse Network.\\
\textbf{Out-of-domain Detection} This kind of research mainly focuses on how to design appropriate scoring functions to detect OOD.~\citet{DBLP:conf/iclr/HendrycksG17} adopt the maximum softmax probability (\textbf{MSP}) and provide several baselines for the follow-up research.~\citet{DBLP:conf/iclr/LiangLS18} (\textbf{ODIN}) add small perturbations to inputs and temperature to softmax score based on maximum softmax probability.~\citet{DBLP:journals/corr/abs-1807-03888} detect OOD samples by calculating the \textbf{Mahalanobis distance} between the sample and the different In-domain distributions.~\citet{9052492} distinguish IND and OOD by \textbf{Entropy} calculated on softmax probability.
~\citet{DBLP:journals/corr/abs-2010-03759} regard the \textbf{Energy} score calculated from output logits as a better scoring function.~\citet{DBLP:conf/nips/SunGL21} propose a simple~\textit{post hoc} OOD detection method by rectifying the activations (\textbf{ReAct}) output in the penultimate layer of model.~\citet{hendrycks2019scaling} propose that under large-scale and real-world settings, taking \textbf{MaxLogit} as the scoring function is better than maximum softmax probability. \\

\textbf{Sparse Network} Our approach is also inspired by the work related to sparse networks.~\citet{DBLP:journals/corr/abs-1712-01312} prune the network by adding $\mathcal{L}_{0}$ norm regularization on parameters.~\citet{DBLP:conf/iclr/FrankleC19} propose the Lottery Ticket Hypothesis--a subnetwork (winning ticket) with comparable performance as the original network can be uncovered from the randomly initialized overparameterized network at the same (or no more than) cost of the training original network.~\citet{DBLP:conf/icml/ZhangAX0C21,zheng-etal-2022-robust} propose that the structure of the model is related to the spurious correlation and an unbiased substructure can be found from the biased model. Based on~\citet{DBLP:journals/corr/abs-1712-01312},~\citet{cao-etal-2021-low} can search for various subnetworks that perform the various linguistic tasks of interest.

\section{Proposed Method}
\subsection{Problem Statement}
OOD intent classification usually adheres to the following paradigm: Denote $Y:=\{1,$\dots$,k\}$ as the pre-defined intent set in the TODS where $k$ is the number of intents and $\mathbb{X}$ as the whole input space. For an utterance $x \in \mathbb{X}$, the logits about intents can be output through a neural network $\mathcal{F}:\mathbb{X} \to R^{[Y]}$. 
 A desirable scoring function (also known as decision function) $\mathcal{G}$, which can detect OOD intent while ensuring the accuracy of the identification of known intents, is the objective of OOD intent classification. The prediction can be formed as:
\begin{equation}
    \hat{Y} = \left\{ 
    \begin{aligned}
    & \text{OOD}, & \mathcal{G}(x, \mathcal{F}) < \theta, \\
    & \text{argmax}_{k \in [Y]}\phi_{k}(\mathcal{F}(x)), & \mathcal{G}(x, \mathcal{F}) \ge \theta.
    \end{aligned}
    \right.,
\end{equation}
where $\phi$ is a function of logits (e.g., Softmax). The threshold $\theta$ is used to distinguish IND ($>= \theta$) and OOD ($< \theta$) according to the scores of the decision function. The typical selection of threshold value needs to ensure high accuracy (e.g., 95\%) of identifying IND.

\subsection{Lottery Tickets Less Overconfident}
\label{overconfident-model}
\citet{DBLP:conf/iclr/FrankleC19} put forward the Lottery Ticket Hypothesis (LTH). In short, a winning ticket $S^{+}$ related to the target task can be identified from a randomly initialized overparameterized network $\Omega$, and the remaining of the network is denoted as $S^{-}$. The relationship between the posterior modeled by the overparameterized network and the posterior modeled by the winning ticket can be formulated as:
\begin{equation}
    \small
    \begin{aligned}
        p(Y|X, \Omega, \xi) &= p(Y|X, S^{\tiny{+}}, S^{\tiny{-}},\xi) \\
                            &= p(Y|X, S^{+}, \xi)\cdot \frac{\tiny{p(S^{-}|X,Y,S^{+},\xi)}}{\tiny{p(S^{-}|X,S^{+},\xi)}},                        
    \label{eq:overparam}
    \end{aligned}
\end{equation}
where $\xi$ is metadata, including target task, training methods, datasets, etc.

Take a closer look at the right hand of Eq.(\ref{eq:overparam}). Given that $S^{+}$, $S^{-}$ are structurally linked and jointly optimize the objective loss supervised by $Y$, it is crucial to note that $S^{-}$ and $Y$ are often actually not independent, but rather often establish a certain spurious positive correlation further exacerbated by the intrinsic bias brought by annotation in the training set, which can be expressed as:
\begin{equation}
    p(S^{-}|X,Y,S^{+},\xi) >= p(S^{-}|X,S^{+},\xi).
    \label{overconfidentforumn}
\end{equation}
We also provide a heuristic proof below.

According to Bayes' theorem, the $p(S^{-}|X,S^{+},\xi)$ can be calculated as follows:
 \begin{equation}
    \small
    p(S^{-}|X,S^{+},\xi) = \sum_{T \in \mathcal{T}} p(S^{-}|X,T,S^{+}, \xi)p(T|X, S^{+}, \xi),
    \label{p-s}
\end{equation}
where $\mathcal{T}$ is the space of target tasks related to dataset $D$ contained in $\xi$, $X$ is the input space of samples in $D$, $Y$ is the specific target task defined by $D$ ($Y \in \mathcal{T}$). The definition of parameters remains consistent with the previous context.

Since $D$ is generally collected for a specific type target task, i.g.,$Y$, $X$ could not act on other type tasks $T \in \mathcal{T} - \{Y\}$ and $S^+$ is a subnetwork defined by $Y$ according to the previous condtions, it can be inferred as follows:
 \begin{equation}
    \forall T \in \mathcal{T} - \{Y\}, p(T|X, S^{+}, \xi) \rightarrow 0.
    \label{zero-task}
\end{equation}
Therefore, take Eq.\eqref{zero-task} into Eq.\eqref{p-s} to get the following expression:

\begin{equation}
    \begin{aligned}
    p(S^{-}|*) &= \underbrace{p(S^{-}|X,Y,S^{+}, \xi)p(Y|X, S^{+}, \xi)}_{Y \in \mathcal{T}} \\
                         &+ \underbrace{0 + \cdots + 0 + \cdots}_{Y \in  \mathcal{T} - \{Y\}} \\
               &= p(S^{-}|X,Y,S^{+}, \xi)p(Y|X, S^{+}, \xi)      
    \end{aligned}
    \label{con}
\end{equation}
According to the Eq.\eqref{con}, the following can be obtained:
\begin{equation}
    \begin{aligned}
    p(S^{-}|X,Y,S^{+}, \xi) &= \frac{p(S^{-}|X,S^{+},\xi)}{p(Y|X, S^{+}, \xi)}\\
                            &>= p(S^{-}|X,S^{+},\xi)             
    \end{aligned}
    \label{con-1}
\end{equation}
Therefore, Eq.(\ref{eq:overparam}) can be further calculated as:
 \begin{equation}
    p(Y|X, S^{\tiny{+}}, S^{\tiny{-}},\xi)  >= p(Y|X, S^{+}, \xi).
\end{equation}
The above expression shows that the overparameterized network prefers to be more overconfident than the winning ticket, which results in giving high confidence to OOD samples, making it difficult to distinguish between IND and OOD.
\subsection{The Road to Open-would Lottery Tickets}
\label{find-lottery-ticket}
It is worth noting that the origin Lottery Ticket Hypothesis is only suitable for the closed-world (i.e., the training and test set come from the same distribution). Further, we extend this hypothesis to the open-world setting--Through one-shot pruning and minor post-processing, we can find luckier winning tickets, which can not only ensure the accuracy of IND intent identification but also better detect OOD intent with the original initialization.\\
\textbf{Backbone and IND Identification} We choose Pre-Trained Model BERT \cite{DBLP:conf/naacl/DevlinCLT19}, represented by $\mathcal{F}(x;\theta)$ with initialization $\theta_{0}$, as the backbone network. To enable the model to effectively identify IND intent, we finetune $\mathcal{F}(x;\theta)$ under the supervision of softmax cross-entropy as suggested in~\citet{zhou-etal-2022-knn}. 
The objective $\mathcal{L}_{ce}$  can be formed as:
\begin{align}
    \label{eq:ce}
    \mathcal{L}_{\text{ce}}(\theta) =  -\frac{1}{N}\sum_{i=1}^{N}\log\frac{\exp(\mathcal{F}_{y_{i}}(z_i))}{\sum_{j=1}^{[Y]}\exp(\mathcal{F}_{j}(z_i))},
\end{align}
where $y_i$ is the label of sample $z_i$, \begin{math}\mathcal{F}_{j}(z_i)\end{math} denotes the logit of the $j^{th}$ class and $\theta$ denotes parameters.\\
\textbf{Seek Parameters Need to be Masked} Different from Iterative Magnitude Pruning (IMP), we generate lottery tickets through one-shot pruning. To this end, inspired by~\citet{DBLP:journals/corr/abs-1712-01312}, we also add a binary ``gate'' to each parameter in the model to determine whether the parameter is of interest to the target task. Specifically, with a Pre-Trained Model $\mathcal{F}(x;\theta)$ at hand, the subnetwork is generated by $\mathcal{F}(x;\theta \odot M)$, where $M \in \{0, 1\}^{|\theta|}$ denotes the ``gates''and $|\theta|$ is the size of the parameters.

However, due to $M$ being a discrete (non-differentiable) binary and the exponential combinatorial property of $2^{|\theta|}$, it cannot be optimized normally. Following~\citet{DBLP:journals/corr/abs-1712-01312}, we put Bernoulli distribution over the entry $m_{i} \sim Bern(\pi_{i})$, where $m_{i} \in M$ and $\pi_{i} = Pr(m_{i}=1)$. In addition to the convenience of optimization, the purpose of introducing random variables is that if the probability of a parameter, i.e., $\pi_{i}$ is too small, we can consider that the parameter is not strongly related to the task, which means it can be ``masked''. The framework of optimization can be defined as:
\begin{align}
    \label{eq-mask-ori}
    \mathcal{L}_{\text{mask}}(\mathlarger{\mathlarger{\bm{\pi}}}) = \text{E}_{q(M;\pi)}[\mathcal{L}_{ce}(\theta \odot M)] + \mathcal{R}(\mathlarger{\bm{\pi}})
\end{align}
where $\mathcal{L}_{ce}$ stands for above loss in Eq.~\eqref{eq:ce} but the parameter to be optimized has changed from $\theta$ to \textbf{$\pi$}. The $\mathcal{R}(\mathlarger{\bm{\pi}})$ is a regularization term w.r.t. parameters $\mathlarger{\bm{\pi}}$. The regularization term can have different forms for different purposes. Here we adopt $L_0$ regularization to encourage sparsity.

To further optimize the first term of Eq.~\eqref{eq-mask-ori} (which cannot be optimized based on gradient due to the discrete nature of $M$), following~\citet{DBLP:journals/corr/abs-1712-01312}, we ``smooth'' the Eq.~\eqref{eq-mask-ori}. With the help of the uniform distribution $\mathcal{U}$(0, 1) and the binary concrete continuous random variable $s_i$ which is distributed in the (0,1) interval, we can reparameterize ($\mathcal{H}$) the $M$ and an entry $m_i \in M$ can be reparameterized as follows:
\begin{align}
    \label{eq:reparam}
    & u_i \sim \mathcal{U}(0, 1),\\
    \label{eq:reparam-1}
    & s_i = \text{Sigmoid}((\log u_i/(1 - u_i) + \alpha_i) / \beta), \\
    & m_i = \min(1, \max(0, s_i(\zeta - \gamma) + \gamma)),
\end{align}
where $(\alpha_i, \beta)$ are the parameters of the binary concrete distribution and $(\zeta < 0, \gamma > 1)$ are constants to stretch the distribution interval of $s_i$. Then objective of Eq.\eqref{eq-mask-ori} can be rewritten as:
\begin{equation}
\begin{aligned}
   \label{eq:mask_reparam}
   \mathcal{L}_{mask} & = \text{E}_{u \in \mathcal{U}(0,1)}[\mathcal{L}_{ce}(\theta \odot \mathcal{H}(u, \bm{\alpha}))] \\
                      & + \lambda \cdot \sum_{i=1}^{|\theta|} \text{Sigmoid}(\alpha_i - \beta\log \frac{-\zeta}{\gamma}),
\end{aligned}
\end{equation}
where $\mathcal{H}$ is above reparameterization and $\lambda$ is a hyper-parameter to balance two terms in $\mathcal{L}_{mask}$. 
In practice, we can adopt Monte Carlo as~\citet{DBLP:journals/corr/abs-1712-01312} to the expectation (i.e. the first term) due to reparameterization.\\ 
\textbf{Retrain with Origin Initialization} After the optimization of Eq.~\eqref{eq:mask_reparam} converges or the iteration reaches a certain number of epochs,
for each parameter, the associated probability $\pi$, which can be considered as the degree of correlation between the parameter and the target task, can be output, and mask can be obtained by $M=\mathbb{I}(\mathlarger{\mathlarger{\mathlarger{\bm{\pi}}}} \ge \mu)$, where $\mathbb{I}$ is indicator function and $\mu$ is the threshold to filter parameters. Finally, assign to the unmasked parameters original initial values in $\theta_{0}$ and retrain the model with new initialization $\hat{\theta}_{0}=\theta_{0}\odot M$.

\subsection{OOD Detection with Lottery Tickets}
\label{lottery-tickets}
To better explore the ability of lottery tickets to detect OOD, we just take maximum softmax probability as scoring function and do not select relatively complex OOD scoring functions, such as \textit{Energy}, \textit{ReAct}, and so on (we also demonstrate our lottery tickets can be well compatible with these downstream detection functions in Section~\ref{compatible-scoring}). However, we will carry out temperature scaling on the logits before that.
We will further demonstrate the effectiveness of temperature scaling in theory.\\
\textbf{What is Temperature Scaling?} The temperature scaling is just a simple extension of the softmax score. Its definition is as follows:
\begin{align}
    \label{eq:tempreature}
    & \mathcal{S}_{i}(x;T) = \frac{\exp(\phi_{i}(x)/T)}{\sum_{j=1}^{k}\exp(\phi_{j}(x)/T)},
\end{align}
where T (usually T $>$ 1) is called the temperature. In Section~\ref{temperature-scaling}, we will analyze it in detail.\\
\textbf{Why is Temperature Scaling?} In addition to calibration, we observed an interesting and common phenomenon (also mentioned in the computer vision field~\cite{hendrycks2019scaling,DBLP:conf/iclr/LiangLS18}). For a pair of indistinguishable IND and OOD samples, excluding the maximum logit score, we find that the remaining logit scores for the IND sample are more uneven (see strict measurement in following prove) than that for the OOD sample.
The intrinsic lie in that the general characteristics of the intent of the IND sample are similar to another (or more) intent and significantly different from other intents. 
Those similar intents will be assigned high confidence, while other (different) intents would be given relatively low confidence, especially when the number of intents is large~\cite{hendrycks2019scaling}. This will cause the confidence of ground truth intent to be dispersed by similar intents, making the model misidentify as OOD.

Different from the previous empirical demonstration, 
in the following proposed theorem, we theoretically demonstrate why temperature scaling (just needs to be greater than 1) can differentiate In- and Out-of-Domain based on the above properties and bring new insights into Temperature Scaling. 
\begin{theorem}
\label{theorem}
Let $x_{A}\in D_{\text{IND}}$ and $x_{B}\in D_{\text{OOD}}$ be from IND and OOD respectively, the logits outputed by pre-trained model $\mathcal{F}$ are $\bm{\phi}_{A}=\{a_1,...,a_k\}$ and $\bm{\phi}_{B}=\{b_1,...,b_k\}$ respectively. Suppose $a_1=\max\phi_{A}$ and $b_1=\max\phi_{B}$ and the probabilities of both are equal after softmax, i.e., $S_{1}(x_{A};T=1)$=$S_{1}(x_{B};T=1)$. Under the condition that the distribution of $\phi_{A}-\{a1\}$ is more uneven than that of $\phi_{B}-\{b1\}$, after temperature scaling, $S_{1}(x_{A};T>1) \geq S_{1}(x_{B};T>1)$.
\end{theorem}
\begin{proof}
According to conditions \\
$\phi_{A}$=$\{a_1,a_2,\cdots,a_k\}$, $\phi_{B}$=$\{b_1,b_2,\cdots,b_k\}$, and $S_{1}(x_{A};T=1)$=$S_{1}(x_{B};T=1)$ we have:
\begin{align}
    \label{eq:assumption-app}
    \frac{\exp(a_1)}{\sum_{j=1}^{k}\exp(a_j)}=\frac{\exp(b_1)}{\sum_{j=1}^{k}\exp(b_j)}
\end{align}
We can get equality by Eq.~\eqref{eq:assumption-app} as:
\begin{align}
    \label{eq:tran-ass-1}
    \mathcal{A}(\phi_{A})&=\mathcal{A}(\phi_{B});\\
    \mathcal{A}(\phi_{A})&=\sum_{j=2}^{k}\exp(a_j-a_1);\\
    \mathcal{A}(\phi_{B})&=\sum_{j=2}^{k}\exp(b_j-b_1).
\end{align}
Now Let us consider introducing Temperature Scaling T(>1), $\mathcal{A}(\phi_{A})$ and $\mathcal{A}(\phi_{B})$ become as:
\begin{align}
    \label{eq:tran-ass}
    \mathcal{A}(\phi_{A},T)&=\sum_{j=2}^{k}(\exp(a_j-a_1))^{\frac{1}{T}};\\
    \mathcal{A}(\phi_{B},T)&=\sum_{j=2}^{k}(\exp(b_j-b_1))^{\frac{1}{T}}.
\end{align}
According to properties of inequalities in~\citet{chen2014brief}, $\sum_{j=2}^{k}\exp(x_j)^{\frac{1}{T}}$ is \textbf{concave} and take maximum value when $\{x_j\}$ is even (equal with each other) denoted as $\bar{X}$.\\
And since the distribution of $\phi_{A}-\{a1\}$ is more uneven than that of $\phi_{B}-\{b1\}$, which can be fomulated as: $\phi_{A}-\{a1\} \in \sigma(\bar{X})$ and $\phi_{B}-\{b1\}$ is out of the range of $\sigma(\bar{X})$ ($\sigma$ is local space spanned by $\bar{X}$ as the center). Combining the properties of \textbf{concave}, we can get $\mathcal{A}(\phi_{A},T)$ $\leq$ $\mathcal{A}(\phi_{B},T)$ and also have:
\begin{align}
    \label{eq:assumption}
    \sum_{j=2}^{k}(\exp(a_j-a_1))^{\frac{1}{T}} &\leq \sum_{j=2}^{k}(\exp(b_j-b_1))^{\frac{1}{T}};\\
    \frac{\exp(a_1)^{\frac{1}{T}}}{\sum_{j=1}^{k}\exp(a_j)^{\frac{1}{T}}}&\geq \frac{\exp(b_1)^{\frac{1}{T}}}{\sum_{j=1}^{k}\exp(b_j)^{\frac{1}{T}}}
\end{align}
Then, that is $S_{1}(x_{A};T>1) \geq S_{1}(x_{B};T>1)$. 
\end{proof}
According to the above full proof, the lead-in of temperature can make full use of such properties, which can effectively cope with such a dilemma.\\
\textbf{Scoring function} Based on the above calibrated softmax score, the definition of score function $\mathcal{G}$ we adopted is as:
\begin{align}
    \label{eq:scoring}
    & \mathcal{G}(x, \mathcal{F}) = \max_{i}\{\mathcal{S}_{i}(x;T)\}.
\end{align}

When the score $\mathcal{G}$ of an utterance is less than a specific threshold $\theta$, it can be regarded as OOD, otherwise, it is IND. As mentioned above, the selection of the threshold needs to ensure the accuracy of the IND. Refer to Section~\ref{evaluation-metrics} for specific metrics.

\section{Experiments}
\subsection{Datasets}
To exhibit the effectiveness and universality of detecting OOD in lottery tickets, we extensively experiment and analysis on three used widely and challenging real-world datasets. \\
\textbf{CLINC-FULL}~\cite{DBLP:conf/emnlp/LarsonMPCLHKLLT19} is dataset that has been annotated and refined manually for evaluating the ability of OOD detection. It has 150 different intents covering 10 various domains and contains 22500 IND samples, 1200 OOD samples respectively. \\
\textbf{CLINC-SMALL}~\cite{DBLP:conf/emnlp/LarsonMPCLHKLLT19} is a variant version of CLINC-FULL and is to measure the ability of OOD dectection of model in the case of insufficient samples. The data also has 150 intents, but each type contains only 50 samples.\\
\textbf{StackOverflow}~\cite{xu2015short} is a public corpus from Kaggle.com. The dataset involves 20 intents, in which the training set, validation set, and test set contain 12000, 2000, and 6000 samples respectively.\\
\textbf{BANKING}~\cite{casanueva2020efficient} It is a dataset about bank-related businesses. Its character is that the number of samples in each category of the dataset is different. The data set includes 77 different categories and the training, and test sets contain 9003, and 3080 respectively. In addition, the validation set also contains 1000 samples.

\begin{table*}[!th]
    \centering
    \begin{tabular}{l | c c c c | c c c c}
    \toprule
     \multirow{3}{*}{\textbf{Methods}} & \multicolumn{4}{c}{\textbf{Clinc-Full}} & \multicolumn{4}{c}{\textbf{Clinc-Small}}\\
    \cmidrule{2-5} \cmidrule{6-9} 
    ~ & ACC & TNR95 & AUROC & AVG. & ACC & TNR95 & AUROC & AVG. \\
    \midrule
    MSP 
    & 91.48\textsubscript{\tiny{0.18}} & 82.27\textsubscript{\tiny{0.78}} & 95.68\textsubscript{\tiny{0.21}} & 89.81
    & 90.19\textsubscript{\tiny{0.06}} & 79.10\textsubscript{\tiny{0.79}} & 95.01\textsubscript{\tiny{0.36}} & 88.10 \\
    MaxLogit 
    & 91.97\textsubscript{\tiny{0.12}} & 85.47\textsubscript{\tiny{0.53}} & 96.02\textsubscript{\tiny{0.26}} & 91.15
    & 90.86\textsubscript{\tiny{0.01}} & 82.87\textsubscript{\tiny{0.33}} & 95.73\textsubscript{\tiny{0.46}} & 89.82 \\
    Energy 
    & 92.01\textsubscript{\tiny{0.18}} & 85.73\textsubscript{\tiny{0.74}}& 96.08\textsubscript{\tiny{0.26}} & 91.27
    & 90.98\textsubscript{\tiny{0.13}} & 84.00\textsubscript{\tiny{0.22}} & 95.83\textsubscript{\tiny{0.47}} & 90.27 \\
    Entropy 
    & 91.61\textsubscript{\tiny{0.20}} & 83.07\textsubscript{\tiny{0.99}} & 95.96\textsubscript{\tiny{0.22}} & 90.21
    & 90.70\textsubscript{\tiny{0.02}} & 82.13\textsubscript{\tiny{0.73}} & 95.41\textsubscript{\tiny{0.38}} & 89.41 \\
    ODIN 
    & 91.99\textsubscript{\tiny{0.10}} & 85.60\textsubscript{\tiny{0.45}} & 96.11\textsubscript{\tiny{0.22}} & 91.23
    & 90.92\textsubscript{\tiny{0.07}} & 83.23\textsubscript{\tiny{0.12}} & 95.87\textsubscript{\tiny{0.38}} & 90.01 \\
    Mahalanbis 
    & 91.94\textsubscript{\tiny{0.23}} & 84.90\textsubscript{\tiny{1.14}} & 96.79\textsubscript{\tiny{0.14}} & 91.21
    & 90.52\textsubscript{\tiny{0.22}} & 81.10\textsubscript{\tiny{0.99}} & 96.26\textsubscript{\tiny{0.04}} & 89.29 \\ 
    \midrule
    \rowcolor{mygray}
    \textbf{OLT(Ours)} & 
    \textbf{92.30}\textsubscript{\tiny{0.10}} & \textbf{86.90}\textsubscript{\tiny{0.50}} & \textbf{96.82}\textsubscript{\tiny{0.14}} & 
    \textbf{92.01} &
    \textbf{91.22}\textsubscript{\tiny{0.07}} & \textbf{84.53}\textsubscript{\tiny{0.25}} & \textbf{96.32}\textsubscript{\tiny{0.09}} &
    \textbf{90.69}\\
    \bottomrule
    \end{tabular}
    \begin{tabular}{l | c c c c | c c c c}
    \toprule
     \multirow{3}{*}{\textbf{Methods}} & \multicolumn{4}{c}{\textbf{Banking}} & \multicolumn{4}{c}{\textbf{Stackoverflow}}\\
    \cmidrule{2-5} \cmidrule{6-9} 
    ~ & ACC & TNR95 & AUROC & AVG. & ACC & TNR95 & AUROC & AVG. \\
    \midrule
    MSP 
    & 79.25\textsubscript{\tiny{1.25}} & 43.73\textsubscript{\tiny{3.95}} & 85.42\textsubscript{\tiny{3.74}} & 69.47
    & 74.95\textsubscript{\tiny{1.15}} & 32.62\textsubscript{\tiny{2.09}} & 89.66\textsubscript{\tiny{0.71}} & 65.74 \\
    MaxLogit 
    & 80.28\textsubscript{\tiny{2.28}} & 48.73\textsubscript{\tiny{7.96}} & 86.40\textsubscript{\tiny{1.41}} & 71.80
    & 75.25\textsubscript{\tiny{1.27}} & 33.47\textsubscript{\tiny{2.38}} & 90.08\textsubscript{\tiny{0.94}} & 66.27 \\
    Energy 
    & 79.71\textsubscript{\tiny{3.00}} & 47.11\textsubscript{\tiny{10.45}}& 86.07\textsubscript{\tiny{1.73}} & 70.96
    & 75.01\textsubscript{\tiny{1.41}} & 32.58\textsubscript{\tiny{3.15}} & 90.11\textsubscript{\tiny{0.97}} & 65.90 \\
    Entropy 
    & 80.18\textsubscript{\tiny{1.29}} & 47.67\textsubscript{\tiny{3.92}} & 85.95\textsubscript{\tiny{3.61}} & 71.27
    & 75.36\textsubscript{\tiny{0.98}} & 33.85\textsubscript{\tiny{1.26}} & 89.93\textsubscript{\tiny{0.65}} & 66.38 \\
    ODIN 
    & 80.33\textsubscript{\tiny{2.46}} & 49.12\textsubscript{\tiny{8.70}} & 86.33\textsubscript{\tiny{1.62}} & 71.93
    & 75.30\textsubscript{\tiny{1.14}} & 33.56\textsubscript{\tiny{1.83}} & 90.36\textsubscript{\tiny{0.55}} & 66.41 \\
    Mahalanbis 
    & 78.84\textsubscript{\tiny{1.71}} & 43.20\textsubscript{\tiny{4.95}} & 88.31\textsubscript{\tiny{2.20}} & 70.12
    & 75.19\textsubscript{\tiny{0.41}} & 33.62\textsubscript{\tiny{0.97}} & 90.71\textsubscript{\tiny{0.78}} & 66.51 \\ 
    \midrule
    \rowcolor{mygray}
    \textbf{OLT(Ours)} & 
    \textbf{82.89}\textsubscript{\tiny{0.94}} & \textbf{58.51}\textsubscript{\tiny{3.89}} & \textbf{89.26}\textsubscript{\tiny{1.52}} & 
    \textbf{76.89} &
    \textbf{75.92}\textsubscript{\tiny{1.16}} & \textbf{35.53}\textsubscript{\tiny{2.28}} & \textbf{91.36}\textsubscript{\tiny{0.34}} &
    \textbf{67.60}\\
    \bottomrule
    \end{tabular}
    \caption{\textbf{Main Results} of comparison between Open-world Lottery Ticket (OLT) and other competitive OOD detection algorithms. \textbf{ACC} is used to measure the overall performance of the model, including both OOD detection and the identification of IND specific class. All reported results are percentages and mean by conducting with different seeds (The subscripts are the corresponding standard deviations).}
    \label{tab-main-result}
\end{table*}

\subsection{Evaluation Metrics and Baselines}
\label{evaluation-metrics}
For all the above datasets, we treat all OOD samples as one rejected class as following previous works~\cite{DBLP:conf/iclr/LiangLS18,zhou-etal-2022-knn}. To evaluate the performance of our method fairly, we follow previous work~\cite{DBLP:conf/iclr/LiangLS18,DBLP:conf/nips/SunGL21} and adopt two widely used metrics: \\
\textbf{TNR at 95\% TPR (TNR95)} (TNR is short for ture negative rate) is to measure the probability that OOD is correctly detected correctly when the true positive rate (TPR) is up to 95\%. \\
\textbf{Area Under the Receiver Operating Charateristic curve (AUROC)} is a threshold-free metric, which reflects the probability of OOD being recognized as OOD is greater than that of IND. A greater value suggests better performance.\\
\textbf{ACCURACY (ACC)} In addition, to better evaluate the overall performance of our method, that is, in addition to detecting OOD, it should effectively identify the specific class of IND. Therefore, We also introduce ACC for all categories.

We extensively compare our method with as many competitive OOD detection algorithms (scoring functions) as possible. 
The entire baseline can be roughly grouped into the following categories:~\textbf{MSP}~\cite{DBLP:conf/iclr/HendrycksG17}, \textbf{MaxLogit}~\cite{hendrycks2019scaling}, \textbf{Energy}~\cite{DBLP:journals/corr/abs-2010-03759}, \textbf{Entropy}~\cite{9052492} are functions of logits. \textbf{ODIN}
~\cite{DBLP:conf/iclr/LiangLS18} are functions of calibrated logits and \textbf{Mahalanbis distance}~\cite{DBLP:journals/corr/abs-1807-03888} is a function of feature. All baselines are introduced in Section~\ref{related-work}. For a fair comparison, the network backbone (\textbf{BERT}) and training loss function (\textbf{Cross-Entropy loss}) of all methods are consistent. All methods do not use or construct additional OOD samples during the training process.

\subsection{Experimental Setting}
\label{sec:expri-setting}
For data preprocessing, we follow previous work~\cite{zhou-etal-2022-knn}. For the dataset Banking and Stackoverflow (the datasets do not contain a specified OOD class), we randomly select 75\% of the whole intent classes as IND, get rid of other classes (remaining 25\%) in the train set (also in verification set), and unify the abandon classes as OOD in the test set. For Clinc-Full and Clinc-Small, we use the specific OOD class included in the dataset itself without additional processing. During the training, we do not utilize any prior knowledge about OOD.

For the network backbone, we use the BERT (bert-uncased, with 12-layer transformer block) provided by Huggingface Transformers.
The parameters we used are also widely recommended. We used an AdamW optimizer with a batch size of 32 and tried learning late in $\{1e-5, 2e-5, 5e-5\}$. In the finetune stage, we trained BERT for 30 epochs. During retraining subnetwork, we tried epochs in $\{15, 20, 30\}$(less than epochs in finetuning).
In practice, satisfactory performance can be achieved by just masking the parameters of specific layers. To train efficiently and achieve better performance, we introduce a hyper-parameter to help specify which layer parameters need to be masked. All experiments are conducted in the Nvidia GeForce RTX-2080 Graphical Card with 11G graphical memory.

\section{Main Results}
\textbf{Main Results} Table~\ref{tab-main-result} shows the comparison of the lottery ticket uncovered from \textbf{BERT} and other competitive OOD detection methods on different datasets. The highlighted results are the best and demonstrate our method can be better than other methods on different datasets and metrics. The results also show that our method can not only ensure the identification of IND but also detect OOD more effectively. At the same time, it can be seen from the above baselines that some detection methods, such as \textit{Energy}, are very competitive. In subsequent experiments, we found that the combination of the lottery ticket and these methods can also achieve better results than the original, which further verifies our proposed Open-world Lottery Ticket Hypothesis. All reported results are average by conducting at least three rounds with different seeds.

\begin{figure}[!ht]
    \centering
    \subfigure{
    \includegraphics[width=0.95\linewidth]{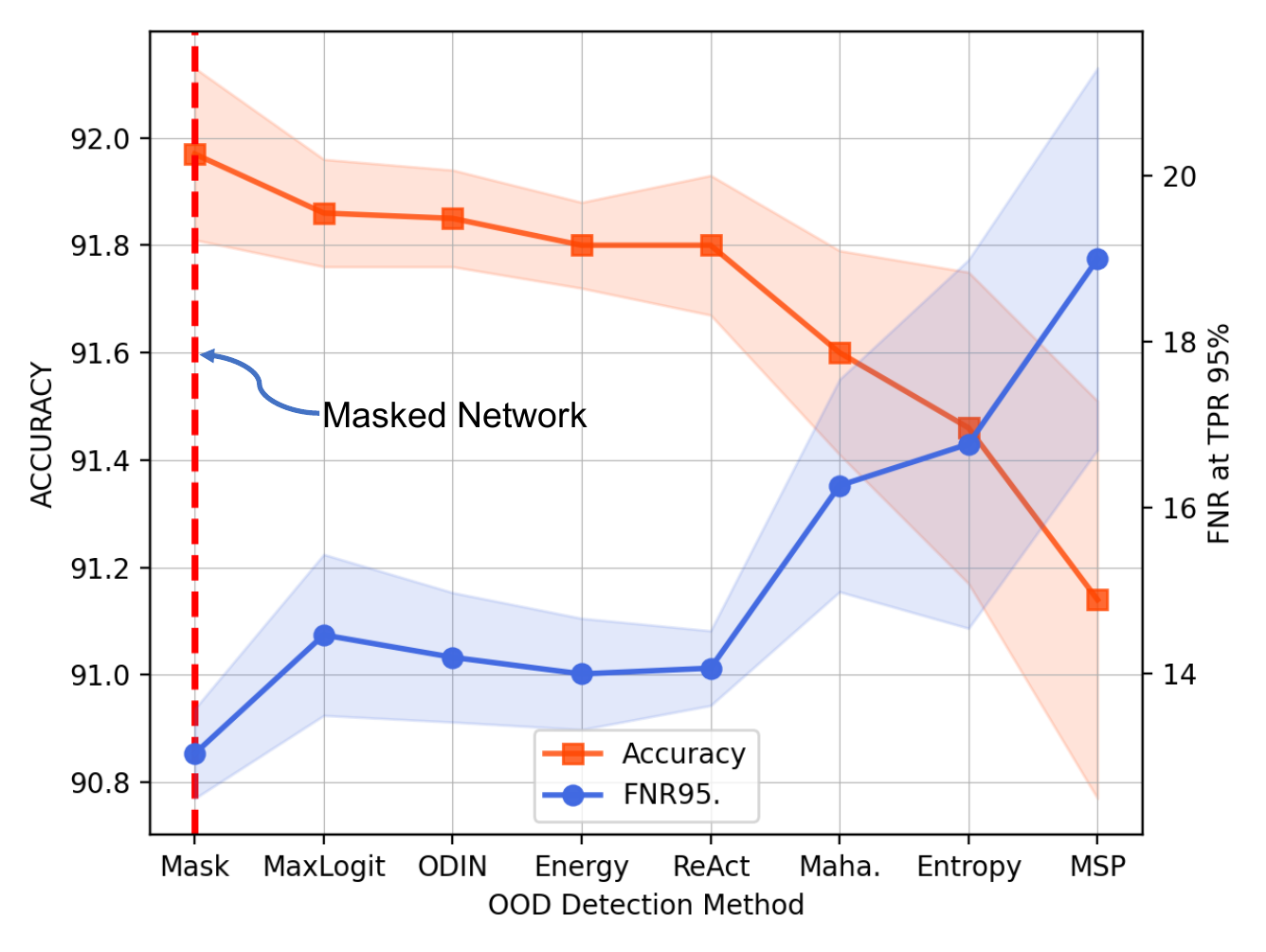}}
    \caption{Masked Network vs. Baselines (Clinc-Full). The red dotted line marks the performance of masked network (subnetwork). Left Y-axis represents the accuary and Right Y-axis represents the FNR@95\%TPR (The lower the value, the better).}
    \label{Mask-only}
\end{figure}

\section{Analysis and Discussions}
\subsection{A Mask for OOD Intent Classification}
\label{mask-model}
In the above process of finding lottery tickets, we need to retrain the subnetwork after resetting unmasked parameters to the original initialization. According to the previous analysis, the subnetwork could also be better calibrated than the original network. Can we get competitive results without retraining? We have also explored this. As suggested in~\citet{DBLP:journals/corr/abs-1712-01312}, we obtain mask $M$ and an entry $m_i\in M$ calculated as:
\begin{align}
    \label{eq:mask}
    & m_i = \min(1, \max(0, \sigma(\log\alpha_{i})(\zeta - \gamma) + \gamma)),
\end{align}
where $\alpha_{i}$ is the parameters in random variable $s_i$ in Eq.~\eqref{eq:reparam-1} , $\sigma$ is Sigmoid function. This operation can also be regarded as $\mathcal{L}_{0}$ norm regularization constraint on parameters and see~\citet{DBLP:journals/corr/abs-1712-01312} for details.

As shown in Fig~\ref{Mask-only}, we can also uncover a masked network (subnetwork) that can effectively detect OOD while maintaining the performance of IND identification. (For a clearer demonstration, we adopt FNR95 metric here to measure OOD detection ability. The lower the value, the better).~\footnote{\text{FNR95} is to measure the probability that OOD is wrongly detected when the TPR is up to 95\%.}The experimental results are consistent with our previous claim. 

Furthermore, masking without retraining can also achieve satisfactory results. \textbf{Is retraining necessary?} Our preferred answer is necessary. We surmise that retraining can learn parameters that are more suitable for the structure of the subnetwork. We hope that our experiments can inspire further theoretical or empirical research.

\begin{table}[!ht]
    \centering
    \setlength\tabcolsep{1.5pt}
    \begin{tabular}{l | c c | c c}
    \toprule
    \multirow{2}{*}{\textbf{Methods}} & \multicolumn{2}{c}{\textbf{Banking}} &\multicolumn{2}{c}{\textbf{Stackoverflow}} \\
    \cmidrule{2-3} \cmidrule{4-5} 
    ~ & ACC & Auroc & ACC & Auroc \\
    \midrule
    MaxLogit 
    & \small{80.28}\textsubscript{2.28} & 
    \small{86.40}\textsubscript{1.41} & 
    \small{75.25}\textsubscript{1.27} & 
    \small{90.08}\textsubscript{0.94} \\ 
    \rowcolor{mygray}
    \textbf{OLT+}\small{MaxLogit} 
    & \small{\textbf{83.03}}\textsubscript{\tiny{0.96}} & \small{\textbf{89.54}}\textsubscript{\tiny{1.67}} & \small{\textbf{75.92}}\textsubscript{\tiny{1.13}} & \small{\textbf{91.32}}\textsubscript{\tiny{0.39}}\\
    \midrule
    Energy 
    & \small{79.71}\textsubscript{3.00} & 
    \small{86.07}\textsubscript{1.73}  & 
    \small{75.01}\textsubscript{1.41} & 
    \small{90.11}\textsubscript{0.97}\\
    \rowcolor{mygray}
    \textbf{OLT+}\small{Energy} 
    & \small{\textbf{82.66}}\textsubscript{\tiny{1.35}} & \small{\textbf{89.27}}\textsubscript{\tiny{1.88}} & \small{\textbf{75.99}}\textsubscript{\tiny{1.08}} & \small{\textbf{91.38}}\textsubscript{\tiny{0.34}}\\
    \midrule
    Entropy 
    & \small{80.18}\textsubscript{1.29} & 
    \small{85.95}\textsubscript{3.61}  & 
    \small{75.36}\textsubscript{0.98} & 
    \small{89.93}\textsubscript{0.65} \\ 
    \rowcolor{mygray}
    \textbf{OLT+}\small{Entropy} 
    & \small{\textbf{82.23}}\textsubscript{\tiny{0.53}} & \small{\textbf{87.28}}\textsubscript{\tiny{1.44}} & \small{\textbf{76.11}}\textsubscript{\tiny{1.06}} & \small{\textbf{91.08}}\textsubscript{\tiny{0.45}} \\
    \midrule
    ODIN 
    & \small{80.33}\textsubscript{2.46} & 
    \small{86.33}\textsubscript{1.62}  & 
    \small{75.30}\textsubscript{1.14} & 
    \small{90.36}\textsubscript{0.55} \\ 
    \rowcolor{mygray}
    \textbf{OLT+}\small{ODIN} 
    & \small{\textbf{82.89}}\textsubscript{\tiny{0.94}} & \small{\textbf{89.26}}\textsubscript{\tiny{1.52}} & \small{\textbf{75.92}}\textsubscript{\tiny{1.16}} & \small{\textbf{91.36}}\textsubscript{\tiny{0.34}} \\
    \midrule
    Mahalabobis 
    & \small{78.84}\textsubscript{1.71} & 
    \small{88.31}\textsubscript{2.20} & 
    \small{75.19}\textsubscript{0.41} & 
    \small{90.71}\textsubscript{0.78} \\
    \rowcolor{mygray}
    \textbf{OLT+}\small{Maha.} 
    & \small{\textbf{81.18}}\textsubscript{\tiny{1.30}} & \small{\textbf{90.34}}\textsubscript{\tiny{1.31}} & \small{\textbf{76.00}}\textsubscript{\tiny{0.69}} & \small{\textbf{91.21}}\textsubscript{\tiny{0.35}} \\
    \bottomrule
    \end{tabular}
    \caption{The Lottery ticket with various OOD Scoring. \textbf{OLT} denotes the backbone is the Open-world Lottery Ticket. Shadow represents our results.}
    \label{tab-scoring-functions}
\end{table}

\subsection{Towards Open-world Lottery Ticket}
\label{compatible-scoring}
To further verify the versatility of the open-world lottery ticket and our extension to the LTH, we demonstrate that the gain of the effect originates from the calibrated network itself. Therefore, we combine the lottery ticket with the various OOD scoring functions and compare performances with the original network. The results are shown in Table~\ref{tab-scoring-functions}. 
From the above results, it can be seen that since the lottery ticket provides calibrated confidence, it can be more compatible with different downstream OOD detection functions and can better differentiate the distribution IND and OOD (showing the value of Auroc is high), especially those related to softmax, such as \textit{Energy} and \textit{MaxLogit}. At the same time, due to differentiation, the lottery network can also better maintain the identification of IND to achieve a higher overall performance (showing the value of ACC is high). These experimental results are consistent with our expectations and can be used as the basis for the establishment of the Open-world Lottery Ticket Hypothesis. 

\begin{table}[!ht]
    \centering
    \setlength\tabcolsep{1.5pt}
    \begin{tabular}{l | c c | c c}
    \toprule
    \multirow{2}{*}{\textbf{Methods}} & \multicolumn{2}{c}{\textbf{Banking}} &\multicolumn{2}{c}{\textbf{Stackoverflow}} \\
    \cmidrule{2-3} \cmidrule{4-5} 
    ~ & ACC & TNR95 & ACC & TNR95 \\
    \midrule
    MaxLogit 
    & \small{76.51}\textsubscript{0.48} & 
    \small{41.14}\textsubscript{4.16} & 
    \small{72.81}\textsubscript{0.53} & 
    \small{32.09}\textsubscript{2.27} \\ 
    \rowcolor{mygray}
    \textbf{OLT+}\small{MaxLogit} 
    & \small{\textbf{76.65}}\textsubscript{\tiny{0.94}} & \small{\textbf{43.02}}\textsubscript{\tiny{2.16}} & \small{\textbf{73.19}}\textsubscript{\tiny{1.07}} & \small{\textbf{32.31}}\textsubscript{\tiny{3.74}}\\
    \midrule
    Energy 
    & \small{\textbf{76.38}}\textsubscript{0.45} & 
    \small{41.27}\textsubscript{4.14}  & 
    \small{72.95}\textsubscript{0.72} & 
    \small{32.85}\textsubscript{2.01}\\
    \rowcolor{mygray}
    \textbf{OLT+}\small{Energy} 
    & \small{76.35}\textsubscript{\tiny{1.43}} & \small{\textbf{42.41}}\textsubscript{\tiny{3.72}} & \small{\textbf{73.39}}\textsubscript{\tiny{1.26}} & \small{\textbf{33.06}}\textsubscript{\tiny{4.72}}\\
    \midrule
    Entropy 
    & \small{\textbf{75.97}}\textsubscript{0.75} & 
    \small{\textbf{38.90}}\textsubscript{5.13}  & 
    \small{72.34}\textsubscript{0.57} & 
    \small{30.44}\textsubscript{2.12} \\ 
    \rowcolor{mygray}
    \textbf{OLT+}\small{Entropy} 
    & \small{75.69}\textsubscript{\tiny{0.56}} & \small{38.73}\textsubscript{\tiny{1.08}} & \small{\textbf{72.96}}\textsubscript{\tiny{0.96}} & \small{\textbf{31.54}}\textsubscript{\tiny{2.93}} \\
    \midrule
    ODIN 
    & \small{76.51}\textsubscript{0.55} & 
    \small{41.23}\textsubscript{4.50}  & 
    \small{72.57}\textsubscript{0.60} & 
    \small{31.27}\textsubscript{2.58} \\ 
    \rowcolor{mygray}
    \textbf{OLT+}\small{ODIN} 
    & \small{\textbf{76.71}}\textsubscript{\tiny{0.96}} & \small{\textbf{43.20}}\textsubscript{\tiny{2.22}} & \small{\textbf{73.26}}\textsubscript{\tiny{1.08}} & \small{\textbf{32.53}}\textsubscript{\tiny{4.01}} \\
    \midrule
    Mahalabobis 
    & \small{7.83}\textsubscript{1.75} & 
    \small{5.88}\textsubscript{2.85} & 
    \small{72.73}\textsubscript{0.67} & 
    \small{31.58}\textsubscript{2.91} \\
    \rowcolor{mygray}
    \textbf{OLT+}\small{Maha.} 
    & \small{\textbf{76.35}}\textsubscript{\tiny{0.49}} & \small{\textbf{40.70}}\textsubscript{\tiny{2.48}} & \small{\textbf{73.24}}\textsubscript{\tiny{0.74}} & \small{\textbf{32.62}}\textsubscript{\tiny{4.22}} \\
    \bottomrule
    \end{tabular}
    \caption{The Open-world Lottery ticket identified from RoBERTa with various OOD scoring functions.}
    \label{tab:resuls-other-model}
\end{table}

\subsection{Generality of Open-World Lottery Ticket}
\label{general-ticket}
In Section~\ref{compatible-scoring}, we have empirically verified Open-world the Lottery Ticket Hypothesis in BERT. In this section, we explore the generality of the Open-world Lottery Ticket Hypothesis and take RoBERTa~\cite{DBLP:journals/corr/abs-1907-11692} as an example to verify whether the Open-world Lottery Ticket Hypothesis is also valid in other models. First of all, We prune a lottery network (OLT) from RoBERTa according to our proposed method in Section~\ref{find-lottery-ticket}.
Then, we replace different \textit{post hoc} scoring functions and make a comprehensive comparison, as we did in Section~\ref{compatible-scoring}. The results are shown in Table~\ref{tab:resuls-other-model}. From the table, we can see that the lottery network discovered from RoBERTa can be also compatible with various scoring functions.
We have preliminarily verified the generality of OLTH, and we hope that the follow-up work will bring more theoretical and experimental research.

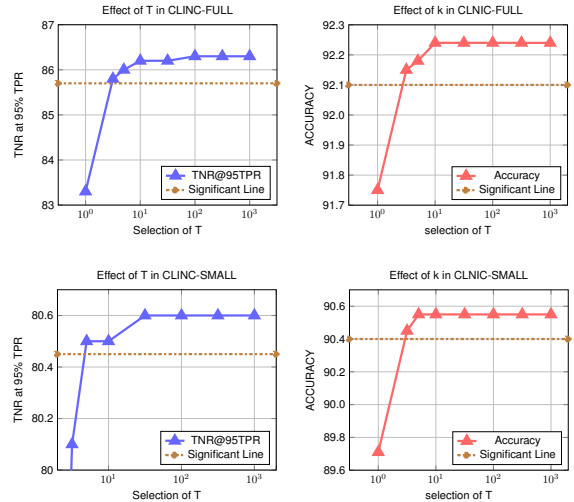
\begin{figure}[!h]
    \centering
    \subfigure{
        \centering
        \begin{tikzpicture}[scale=0.42]
            \begin{axis}[
                xlabel=Selection of T,
                ylabel=TNR at 95\% TPR,
                xmin=-0.5,xmax=3.5,
                ymin=83,ymax=87,
                xtick = {0.0,1.0,2.0,3.0,4.0},
                ytick = {83,84,85,86,87},
                xticklabels={$10^0$,$10^1$,$10^2$, $10^3$,$10^4$},
                yticklabels={83,84,85,86,87},
                grid=major,
                title = Effect of T in CLINC-FULL,
                legend pos=south east
              ]
            \addplot[color=blue!60, line width=2pt, mark=triangle*, mark size=5pt] coordinates {
                (0.0, 83.3)
                (0.5, 85.8)
                (0.7, 86.0)
                (1.0, 86.2)
                (1.5, 86.2)
                (2.0, 86.3)
                (2.5, 86.3)
                (3.0, 86.3)
            };
            \addlegendentry{TNR@95TPR}
            \addplot[color=brown, line width=2pt, mark=square*, mark size=2pt, dash dot] coordinates {
                (-0.5,85.7)
                (3.5,85.7)
            };
            \addlegendentry{Significant Line}
            \end{axis}
        \end{tikzpicture}
    }
    \subfigure{
        \centering
        \begin{tikzpicture}[scale=0.42]
            \begin{axis}[
                xlabel=selection of T,
                ylabel=ACCURACY,
                xmin=-0.5,xmax=3.3,
                ymin=91.7,ymax=92.3,
                xtick = {0.0,1.0,2.0,3.0,4.0},
                ytick = {91.7,91.8,91.9,92.0,92.1,92.2,92.3},
                xticklabels={$10^0$,$10^1$,$10^2$, $10^3$,$10^4$},
                yticklabels={91.7,91.8,91.9,92.0,92.1,92.2,92.3},
                grid=major,
                title = Effect of k in CLNIC-FULL,
                legend pos=south east
              ]
            \addplot[color=red!60, line width=2pt, mark=triangle*, mark size=5pt] coordinates {
                (0.0, 91.75)
                (0.5, 92.15)
                (0.7, 92.18)
                (1.0, 92.24)
                (1.5, 92.24)
                (2.0, 92.24)
                (2.5, 92.24)
                (3.0, 92.24)
            };
            \addlegendentry{Accuracy}
            \addplot[color=brown, line width=2pt, mark=square*, mark size=2pt, dash dot] coordinates {
                (-0.5,92.10)
                (3.3,92.10)
            };
            \addlegendentry{Significant Line}
            \end{axis}
        \end{tikzpicture}
    }
    \subfigure{
        \centering
        \begin{tikzpicture}[scale=0.42]

            \begin{axis}[
                xlabel=Selection of T,
                ylabel=TNR at 95\% TPR,
                xmin=0.3,xmax=3.3,
                ymin=80,ymax=80.7,
                xtick = {0.0,1.0,2.0,3.0,3.3},
                ytick = {80,80.2,80.4,80.6},
                xticklabels={$10^0$,$10^1$,$10^2$, $10^3$},
                yticklabels={80,80.2,80.4,80.6},
                grid=major,
                title = Effect of T in CLINC-SMALL,
                legend pos=south east
              ]
            \addplot[color=blue!60, line width=2pt, mark=triangle*, mark size=5pt] coordinates {
                (0.0, 75.5)
                (0.5, 80.1)
                (0.7, 80.5)
                (1.0, 80.5)
                (1.5, 80.6)
                (2.0, 80.6)
                (2.5, 80.6)
                (3.0, 80.6)
            };
            \addlegendentry{TNR@95TPR}
            \addplot[color=brown, line width=2pt, mark=square*, mark size=2pt, dash dot] coordinates {
                (0.3,80.45)
                (3.3,80.45)
            };
            \addlegendentry{Significant Line}
            
            \end{axis}
        \end{tikzpicture}
    }
    \subfigure{
        \centering
        \begin{tikzpicture}[scale=0.42]

            \begin{axis}[
                xlabel=selection of T,
                ylabel=ACCURACY,
                xmin=-0.5,xmax=3.3,
                ymin=89.6,ymax=90.7,
                xtick = {0.0,1.0,2.0,3.0,4.0},
                ytick = {89.6,89.8,90.0,90.2,90.4,90.6},
                xticklabels={$10^0$,$10^1$,$10^2$, $10^3$,$10^4$},
                yticklabels={89.6,89.8,90.0,90.2,90.4,90.6},
                grid=major,
                title = Effect of k in CLNIC-SMALL,
                legend pos=south east
              ]
            \addplot[color=red!60, line width=2pt, mark=triangle*, mark size=5pt]coordinates {
                (0.0, 89.71)
                (0.5, 90.45)
                (0.7, 90.55)
                (1.0, 90.55)
                (1.5, 90.55)
                (2.0, 90.55)
                (2.5, 90.55)
                (3.0, 90.55)
            };
            \addlegendentry{Accuracy}
            \addplot[color=brown, line width=2pt, mark=square*, mark size=2pt, dash dot] coordinates {
                (-0.5,90.40)
                (3.3,90.40)
            };
            \addlegendentry{Significant Line}
            \end{axis}
        \end{tikzpicture}
    }
    \caption{Effect of Temperature Scaling. As T becomes larger, the benefits brought by T will soon become smaller.}
    \label{fig-effect-of-temperature}
\end{figure}

\subsection{Analysis On Temperature Scaling}
\label{temperature-scaling}
In Theorem~\ref{theorem}, we demonstrate that temperature scaling can help differentiate the distribution between IND and OOD. Let us take a closer look at the behavior of temperature here. In the previous work~\cite{DBLP:conf/iclr/LiangLS18}, it is suggested to take a sufficiently larger value of temperature. However, from the proof of Theorem~\ref{theorem}, it can be seen that temperature just needs to be greater than 1.
We choose temperatures at different scales to test the effect (of temperature) on different data sets and the results are shown in Figure~\ref{fig-effect-of-temperature}. We find that after $T>1$ (without a large value), the effect of OOD detection is very significant. As T becomes larger, the benefits brought by T will soon become smaller, which is in line with our expectations.

\section{Conclusion and Future Work}
\label{sec:conclusion}
Does the model know what it does not know? This paper makes an in-depth discussion of the theory and the practice. Firstly, we discuss the reasons that prevent the model from giving trustworthy confidence. 
Then, we uncover a subnetwork from an overparameterized model to provide calibrated confidence (helpful to differentiate in IND and OOD). In addition, We prove that temperature scaling can help distinguish IND and OOD. Combined with calibrated confidence of subnetwork and temperature scaling, we further extend the LTH to the open-world empirically and verify our conjecture by experiments.

In a larger scope, the research of this paper can be categorized more broadly into the knowledge boundary (or capability boundary) of models, which is a fundamental and critical issue in the deep learning field. With the unprecedented prevalence of artificial intelligence in recent years, research on the knowledge boundary of models, especially large-scale pre-trained language models, has drawn strong attention from scholars in academia and industry, and the research scope has become more extensive~\cite{DBLP:journals/corr/abs-2207-05221,yin-etal-2023-large,Cheng2024CanAA}.

First, in terms of model structure, existing research is increasingly focusing on large-scale generative architectures~\cite{Touvron2023Llama2O}. Can the open-world lottery ticket also be found in generative models? According to current research~\cite{azaria-mitchell-2023-internal}, the answer seems to be affirmative. Then, from the perspective of task form, this study focuses on identifying the capability boundary of models. In practical scenarios, it is equally important to extend the capability boundary of models. For this purpose, a class of research~\cite{zhou-etal-2023-towards-open} has extended the task paradigm by collecting corpora that are not within the capability boundary of the model after identification. These corpora can be further fine-grained discovered~\cite{Zhang_Xu_Lin_Lyu_2021, zhou-etal-2023-probabilistic} to enhance the capability boundary of the model in practical scenarios.
The proposed open-world lottery ticket mainly aims to enhance the cognitive boundary of the model. How can it further adapt to the extension of the model's capability? 

Finally, as the parameter scale of models becomes increasingly larger, the inference speed of models is gradually becoming a bottleneck for their application in practical scenarios. Existing work~\cite{zhou-etal-2023-two} has preliminarily discovered that model inference optimization can not only improve the inference speed but also maintain their overall performance. How to perform inference optimization based on the \textit{Open-World Lottery Ticket Hypothesis} is also a direction worth paying attention to.

\section*{Acknowledgments}
This work was supported by the National Key Research and Development Program of China (No.2022ZD0160102).  
\nocite{*}
\section{Bibliographical References}\label{sec:reference}

\bibliographystyle{lrec-coling2024-natbib}
\bibliography{lrec-coling2024-example}
\bibliographystylelanguageresource{lrec-coling2024-natbib}
\bibliographylanguageresource{languageresource}

\end{document}